\documentclass[sigplan,nonacm]{acmart}
\AtBeginDocument{%
  \providecommand\BibTeX{{%
    \normalfont B\kern-0.5em{\scshape i\kern-0.25em b}\kern-0.8em\TeX}}}

\setcopyright{acmcopyright}
\copyrightyear{2018}
\acmYear{2018}
\acmDOI{XXXXXXX.XXXXXXX}

\acmConference[Conference acronym 'XX]{Make sure to enter the correct
  conference title from your rights confirmation emai}{June 03--05,
  2018}{Woodstock, NY}
\acmPrice{15.00}
\acmISBN{978-1-4503-XXXX-X/18/06}

\usepackage{pgf}
\usepackage{import}
\usepackage{tikz,tkz-euclide,pdftexcmds,calc}
\usetikzlibrary{arrows,shapes.geometric}
\usetikzlibrary{positioning}
\usetikzlibrary{automata}
\usetikzlibrary{arrows,shapes,calc}
\usepackage{circuitikz}
\usepackage{algorithmic}
\usepackage[linesnumbered,ruled,noend]{algorithm2e}
\usepackage{bbm}

\usepackage{mathtools}

\begin{document}

% \title{Mapping periodic data to hyperspace with application in hyperdimensional computing}

% \title{Circular Hypervectors for learning periodic data with Hyperdimensional Computing}

% \title{Circular Hypervectors for faster generalization on learning with periodic data with Hyperdimensional Computing}

% \title{Basis Hypervector extension for Circular and Periodic data: learning with Hyperdimensional Computing}

\title{An Extension to Basis-Hypervectors for Learning from Circular Data in Hyperdimensional Computing}


\author{Igor Nunes, Mike Heddes, Tony Givargis and Alexandru Nicolau}
\affiliation{
  \institution{Department of Computer Science, University of California, Irvine}
  \city{Irvine}
  \state{California}
  \country{United States of America}
}
\email{{igord,mheddes, givargis, nicolau}@uci.edu}

\renewcommand{\shortauthors}{Nunes and Heddes, et al.}

\begin{abstract}
\textit{Hyperdimensional Computing} (HDC) is a computation framework based on properties of high-dimensional random spaces. It is particularly useful for machine learning in resource-constrained environments, such as embedded systems and IoT, as it achieves a good balance between accuracy, efficiency and robustness. The mapping of information to the hyperspace, named \textit{encoding}, is the most important stage in HDC. At its heart are \textit{basis-hypervectors}, responsible for representing the smallest units of meaningful information. In this work we present a detailed study on basis-hypervector sets, which leads to practical contributions to HDC in general: 1) we propose an improvement for level-hypervectors, used to encode real numbers; 2) we introduce a method to learn from circular data, an important type of information never before addressed in machine learning with HDC. Empirical results indicate that these contributions lead to considerably more accurate models for both classification and regression with circular data.
\end{abstract}





\maketitle

\section{Introduction}
\label{sec:intro}

For some time now, machine learning has largely dictated not just academic research and industrial applications, but aspects of modern society in general. Such aspects range from the widespread recommendation systems, which influence the way we consume products, news and entertainment, to social networks that impact the way we behave and relate. Given such broad importance, the demand for devices capable of learning has spread to resource constrained platforms such as embedded systems and internet of things (IoT) devices~\cite{kaelbling1993learning,branco2019machine}. Such scenarios present obstacles to established methods, designed primarily to operate on powerful servers.

The most notable class of such methods is Deep Learning, which has achieved superhuman performance on certain tasks~\cite{schmidhuber2015deep}. Although initially inspired by the remarkably efficient animal brain, Deep Learning owes much of its high energy and computational cost to neurally implausible design choices~\cite{crick1989recent,marblestone2016toward,whittington2019theories}. The dilemma is that these choices, especially error back-propagation and large depth, are also key drivers of its success~\cite{lecun2015deep}. In this context, the search for alternatives has received significant attention from the scientific community~\cite{james2017historical,deneve2017brain,kanerva1988sparse}.

One emerging alternative is \textit{Hyperdimensional Computing} (HDC)~\cite{kanerva2009hyperdimensional}. Like Deep Learning, HDC is also inspired by neuroscience, but the central observation is that large circuits are fundamental to the brain's computation. Therefore, information in HDC is represented using \textit{hypervectors}, typically 10,000-bit words where each bit is \textit{independently and identically distributed} (i.i.d). The i.i.d relationship between bits, also known as holographic information representation, provides inherent robustness since each bit carries exactly the same amount of information. Furthermore, the arithmetic in HDC, as detailed in Section~\ref{sec:HDC}, is generally dimension-independent, which opens up the opportunity for massive parallelism, providing the efficiency sought in HDC.

HDC has already proven to be useful in several applications, including both learning problems, such as classification~\cite{ge2020classification} and regression~\cite{hersche2020integrating}, and classical problems, such as consistent hashing~\cite{heddes2022hyperdimensional}. Regardless of the application, the most fundamental step in HDC is mapping objects in the input space to the hyperspace, a process called \textit{encoding}. Useful encoding functions have already been proposed for several different types of data, such as time series~\cite{rahimi2016hyperdimensional}, text~\cite{rahimi2016robust}, images~\cite{manabat2019performance} and graphs~\cite{nunes2022hyperdimensional}. All these encodings have one thing in common: they start by encoding simple information (e.g. pixel values, vertices and edges, letters, amplitudes of a signal), which are then combined to represent something complex. In this work we study \textit{basis-hypervectors}, the encoding of these atomic pieces of information.

Basis-hypervectors are a cornerstone of HDC and directly effect the accuracy of learned HDC models, as will be evident from the results in Section~\ref{sec:experiments}. Nevertheless, to the best of our knowledge, this is the first work with a primary focus on the study of basis-hypervectors. We start with a comparative study of the two existing types of basis-hypervectors, \textit{random} and \textit{level-hypervectors}, used respectively to represent uncorrelated and linearly-correlated data. Inspired by information theory, our first contribution is an improved method to create level-hypervectors.

Based on the improved level-hypervectors, our main contribution is a basis-hypervector set for \textit{circular data} in a learning setting. Circular data are derived from the measurement of directions, usually expressed as an angle from a fixed reference direction. In addition, it is common to convert time measurements, such as the hours of a day, to angles. As we will discuss in more detail in Section~\ref{sec:circularData}, circular data is present in many fields of research, including astronomy, medicine, engineering, biology, geology and meteorology~\cite{fisher1995statistical,lee2010circular,mardia2000directional}. Embedded systems and IoT also deal with a wide variety of circular data, such as in robotics where the joints generate angular data~\cite{gao2014jhu}, and satellites that fly in elliptical orbits~\cite{lucas2017machine}. The importance is such that the study and interpretation of this type of data gave rise to a subdiscipline of statistics called \textit{directional statistics} (also known as circular or spherical statistics)~\cite{mardia2000directional,pewsey2021recent}. Despite this great relevance, our work is the first to address learning from circular data in HDC.

To assess the practical impact of these contributions, we conducted experiments with publicly available datasets that contain circular data relevant to real-world applications. We compared the proposed basis set with the two existing ones in both regression and classification tasks. We obtained an improvement of 7.2\% in the classification of 15 types of surgical gestures. In regression, the error is reduced by 67.7\% in temperature and power level prediction tasks.
\section{Hyperdimensional Computing}
\label{sec:HDC}

Hyperdimensional Computing (HDC) is a computation model that relies on very high dimensionality and randomness. Inspired by neuroscience, it seeks to mimic and exploit important characteristics of the animal brain while balancing accuracy, efficiency and robustness~\cite{kanerva2009hyperdimensional}. The central idea is to represent inputs $x \in \mathcal{X}$ by projecting them onto a hyperspace $\mathcal{H}=\{0, 1\}^d$, with $d \approx 10{,}000$ dimensions. This mapping $\phi: \mathcal{X}\to\mathcal{H}$ is called \textit{encoding}, and the resulting representations $\phi(x)$ are named \textit{hypervectors}.

The intuitive principle that guides the design of encoding functions is that similar objects in input space need to be mapped to similar hypervectors in hyperspace. We use the normalized Hamming distance as a measure of distance between hypervectors, which takes the form:
\begin{align*}
    \delta:\mathcal{H}\times\mathcal{H}\to\{\sigma\in\mathbb{R} \mid 0\leq\sigma \leq 1\}    
\end{align*}
We then define the hypervector \textit{similarity} to be $1 - \delta$. All cognitive tasks in HDC are ultimately based on similarity. Predictions or decisions are inferred from a model which is created by transforming and combining information using HDC arithmetic.

\subsection{Operations}
\label{sec:operations}
The arithmetic in HDC is based on a simple set of two element-wise, dimension-independent, operations in addition to a permutation operator~\cite{kanerva2009hyperdimensional}. The three operations, \textit{binding}, \textit{bundling}, and \textit{permuting} are illustrated in Figure~\ref{fig:operations}.

\paragraph{Binding} The binding operation is used to ``associate'' information. The function $\otimes: \mathcal{H} \times \mathcal{H} \to \mathcal{H}$ multiplies two input hypervectors to produces a third vector dissimilar to both the operands. The binding operator is commutative and distributive over bundling and serves at its own inverse, i.e., $A\otimes (A\otimes B) = B$. The binding operation is efficiently implemented as element-wise XOR.

\paragraph{Bundling} The bundling operation is used to represent a set of information. The function $\oplus: \mathcal{H}\times\mathcal{H}\to\mathcal{H}$ performs addition on its inputs to produce an output hypervector that is similar to its operands. Bundling is implemented as an element-wise majority operation. The output then represents the mean-vector of its inputs.

\paragraph{Permuting} The permutation operator is often used to establish an order, that is, to differentiate permutations of a sequence. The function takes the form $\Pi: \mathcal{H} \to \mathcal{H}$, whose output hypervector is dissimilar to the input. The exact input can be retrieved with the inverse operation. The operator performs multiplication with a permutation matrix and the most common is a cyclic shift. With $\Pi^i(A)$ we denote a cyclic shift of the elements of $A$ by $i$ coordinates.

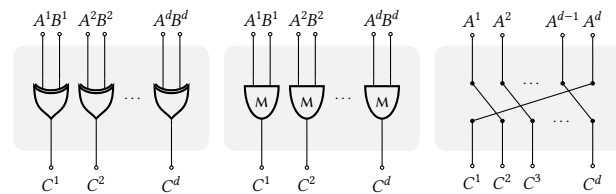
\begin{figure}[h]
    \centering
    \begin{circuitikz}[scale=0.4, transform shape]
    \tikzset{every node}=[font=\sffamily]
    
    % Operation bounding boxes
    \fill [rounded corners=1ex, gray!10](-1.25,-1) rectangle (5.25,2.5);
    \fill [rounded corners=1ex, gray!10](5.75,-1) rectangle (12.25,2.5);
    \fill [rounded corners=1ex, gray!10](12.75,-1) rectangle (19.25,2.5);
    
    % Binding gates
    \node[xor port, rotate=-90] (xor1) at (0, 0)    {};
    \node[xor port, rotate=-90] (xor2) at (1.5, 0)    {};
    \node at (2.75, 0.75) {\LARGE $\dots$};
    \node[xor port, rotate=-90] (xor3) at (4, 0)  {};
    
    % Bundling gates
    \node[and port, rotate=-90] (and1) at (7, 0)    {\LARGE \rotatebox{90} M};
    \node[and port, rotate=-90] (and2) at (8.5, 0)    {\LARGE \rotatebox{90} M};
    \node at (9.75, 0.75) {\LARGE $\dots$};
    \node[and port, rotate=-90] (and3) at (11, 0)  {\LARGE \rotatebox{90} M};
    
    % Permute wires
    \draw (14, 1.25) to[short, *-*] (15, 0);
    \draw (15, 1.25) to[short, *-*] (16, 0);
    \node at (16, 1.25) {\LARGE $\dots$};
    \node at (17, 0) {\LARGE $\dots$};
    \draw (17, 1.25) to[short, *-*] (18, 0);
    \draw (18, 1.25) to[short, *-*] (14, 0);
    
    % % Inputs
    % \draw (1, 3.5) node[above, yshift=1ex]{\huge $A$} to[multiwire={\LARGE d}, o-*] (1, 2.5);
    % \draw (3, 3.5) node[above, yshift=1ex]{\huge $B$} to[multiwire={\LARGE d}, o-*] (3, 2.5);
    
    % \draw (8, 3.5) node[above, yshift=1ex]{\huge $A$} to[multiwire={\LARGE d}, o-*] (8, 2.5);
    % \draw (10, 3.5) node[above, yshift=1ex]{\huge $B$} to[multiwire={\LARGE d}, o-*] (10, 2.5);
    
    % \draw (16, 3.5) node[above, yshift=1ex]{\huge $A$} to[multiwire={\LARGE d}, o-*] (16, 2.5);
    
    % % Outputs
    % \draw (2, -1) to[multiwire={\LARGE d}, *-o] (2, -2) node[below, yshift=-1ex]{\huge $C$};
    
    % \draw (9, -1) to[multiwire={\LARGE d}, *-o] (9, -2) node[below, yshift=-1ex]{\huge $C$};
    
    % \draw (16, -1) to[multiwire={\LARGE d}, *-o] (16, -2) node[below, yshift=-1ex]{\huge $B$};
    
    % Inputs
    \draw (xor1.in 2) to[short, -o] ++(0, 1.5) node[above, yshift=1ex]{\huge $A^1$};
    \draw (xor2.in 2) to[short, -o] ++(0, 1.5) node[above, yshift=1ex]{\huge $A^2$};
    \draw (xor3.in 2) to[short, -o] ++(0, 1.5) node[above, yshift=1ex]{\huge $A^d$};
    
    \draw (xor1.in 1) to[short, -o] ++(0, 1.5) node[above, yshift=1ex]{\huge $B^1$};
    \draw (xor2.in 1) to[short, -o] ++(0, 1.5) node[above, yshift=1ex]{\huge $B^2$};
    \draw (xor3.in 1) to[short, -o] ++(0, 1.5) node[above, yshift=1ex]{\huge $B^d$};
    
    \draw (and1.in 2) to[short, -o] ++(0, 1.5) node[above, yshift=1ex]{\huge $A^1$};
    \draw (and2.in 2) to[short, -o] ++(0, 1.5) node[above, yshift=1ex]{\huge $A^2$};
    \draw (and3.in 2) to[short, -o] ++(0, 1.5) node[above, yshift=1ex]{\huge $A^d$};
    
    \draw (and1.in 1) to[short, -o] ++(0, 1.5) node[above, yshift=1ex]{\huge $B^1$};
    \draw (and2.in 1) to[short, -o] ++(0, 1.5) node[above, yshift=1ex]{\huge $B^2$};
    \draw (and3.in 1) to[short, -o] ++(0, 1.5) node[above, yshift=1ex]{\huge $B^d$};
    
    \draw (14, 1.25) to[short, -o] ++(0, 1.6) node[above, yshift=1ex]{\huge $A^1$};
    \draw (15, 1.25) to[short, -o] ++(0, 1.6) node[above, yshift=1ex]{\huge $A^2$};
    \draw (17, 1.25) to[short, -o] ++(0, 1.6) node[above, yshift=1ex]{\huge $A^{d-1}$};
    \draw (18, 1.25) to[short, -o] ++(0, 1.6) node[above, yshift=1ex]{\huge $A^d$};
    
    % Outputs
    \draw (xor1.out) to[short, -o] ++(0, -1.4) node[below, yshift=-1ex]{\huge $C^1$};
    \draw (xor2.out) to[short, -o] ++(0, -1.4) node[below, yshift=-1ex]{\huge $C^2$};
    \draw (xor3.out) to[short, -o] ++(0, -1.4) node[below, yshift=-1ex]{\huge $C^d$};

    \draw (and1.out) to[short, -o] ++(0, -1.4) node[below, yshift=-1ex]{\huge $C^1$};
    \draw (and2.out) to[short, -o] ++(0, -1.4) node[below, yshift=-1ex]{\huge $C^2$};
    \draw (and3.out) to[short, -o] ++(0, -1.4) node[below, yshift=-1ex]{\huge $C^d$};
    
    \draw (14, 0) to[short, -o] ++(0, -1.5) node[below, yshift=-1ex]{\huge $C^1$};
    \draw (15, 0) to[short, -o] ++(0, -1.5) node[below, yshift=-1ex]{\huge $C^2$};
    \draw (16, 0) to[short, -o] ++(0, -1.5) node[below, yshift=-1ex]{\huge $C^3$};
    \draw (18, 0) to[short, -o] ++(0, -1.5) node[below, yshift=-1ex]{\huge $C^d$};

\end{circuitikz}
    \caption{\label{fig:operations} \textit{Binding}, \textit{bundling}, and \textit{cyclic shift} permutation operations illustrated on binary hypervectors $A$ and $B$ where the superscript denotes the element index of the hypervector. The logical gates in the bundling operation are \textit{majority gates}.} 
\end{figure}

\medskip
Much of HDC's ability to learn comes from the fact that the very high dimension of the $\mathcal{H}$-space allows combining information with these operations while preserving the information of the operands with high probability, due to the existence of a huge number of quasi-orthogonal vectors in the space. For a more theoretical analysis of these properties see the work by Thomas et al.~\cite{thomas2021theoretical}.

\subsection{Classification}
\label{sec:HDCclassification}

Once the encoding function $\phi$ is defined, the training of a classification model with HDC is quite intuitive. An overview of the standard HDC classification framework is illustrated in Figure~\ref{fig:HDCClassification}. For each class $i\in\{1,\dots,k\}$ in the training set, we construct a hypervector $M_i$ as follows:
\begin{align*}
    M_i = \bigoplus_{j :\ell(x_j )=i} \phi(x_j )
\end{align*}
where each $x_j \in {\mathcal X}$ is a training sample and $\ell(x_j ) \in \{1,\dots,k\}$ its respective label. The $\bigoplus$ symbol represents the bundling of hypervectors. The hypervector $M_i$ has the smallest average distance to the hypervectors obtained by encoding the training samples of class $i$. For this reason, each of these hypervectors is used as a ``prototype'' of its respective class and is referred to as a \textit{class-vector}. A trained HDC classification model is therefore denoted by $\mathcal{M} = \{M_i , \dots ,M_k\}$, and simply consists of a class-vector for each class.

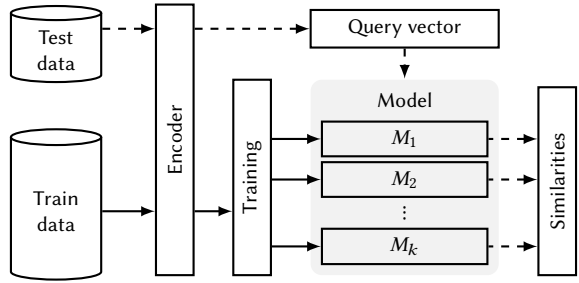
\begin{figure}[ht]
 \centering
 \begin{tikzpicture}[thick]
 
 \tikzstyle{every node}=[font=\footnotesize\sffamily]
 
 \node[
       cylinder,
    %   cylinder uses custom fill,
    %   cylinder body fill=yellow!50,
    %   cylinder end fill=yellow!50,
      shape border rotate=90,
      minimum width=1.2cm,
      aspect=0.25,
      draw, align=center] (test_ds) {Test\\data};
      
     \node[
       cylinder,
    %   cylinder uses custom fill,
    %   cylinder body fill=yellow!50,
    %   cylinder end fill=yellow!50,
      shape border rotate=90,
      minimum width=1.2cm,
      minimum height=2cm, 
      aspect=0.25,
      draw, align=center, below=0.6cm of test_ds] (train_ds) {Train\\data};
      
     \node[draw, minimum height=0.5cm, minimum width=3.6cm, yshift=-0.25cm, rotate=90, right=0.7cm of train_ds.south east, anchor=north west] (enc) {Encoder};
     
     \node[draw, minimum height=0.5cm, minimum width=2.6cm, rotate=90, right=0.5cm of enc.south west, anchor=north west] (train) {Training};
     
     \node[rounded corners=1ex, fill=gray!10, minimum height=2.6cm, minimum width=2.5cm, right=0.5cm of train.south west, anchor=south west] (model) {};
     
     \node[below=0.05cm of model.north, anchor=north] (model_lbl) {Model};
     \node[draw, below=0.05cm of model_lbl.south, anchor=north, minimum width=2.2cm] (m1) {$M_1$};
     \node[draw, below=0.05cm of m1.south, anchor=north, minimum width=2.2cm] (m2) {$M_2$};
     \node[below=-0.05cm of m2.south, anchor=east, rotate=90] (mdot) {$...$};
     \node[draw, below=-0.05cm of mdot.west, anchor=north, minimum width=2.2cm] (mk) {$M_k$};
          
     \node[draw, minimum height=0.5cm, minimum width=2.5cm, above=0.4cm of model.north west, anchor=south west] (query) {Query vector};
     
     \node[draw, minimum width=2.5cm, minimum height=0.5cm, rotate=90, right=0.5cm of model.south east, anchor=north west] (sim) {Similarities};
     
    \path
    (test_ds.east |- query.west) edge[-latex, dashed] (query.west -| enc.north)
    (enc.south |- query.west) edge[-latex, dashed] (query.west)
    (train_ds.east) edge[-latex] (train_ds.east -| enc.north)
    (enc.south |- train_ds.east) edge[-latex] (train_ds.east -| train.north)
    (train.south |- m1.west) edge[-latex] (m1.west)
    (train.south |- m2.west) edge[-latex] (m2.west)
    (train.south |- mk.west) edge[-latex] (mk.west)
    (query.south) edge[-latex, dashed] (model.north)
    (sim.north |- m1.east) edge[latex-, dashed] (m1.east)
    (sim.north |- m2.east) edge[latex-, dashed] (m2.east)
    (sim.north |- mk.east) edge[latex-, dashed] (mk.east)
    ;

\end{tikzpicture}
  \caption{\label{fig:HDCClassification} Overview of the Hyperdimensional Computing classification framework. Solid lines indicate training steps, dashed lines indicate inference steps.}
\end{figure}

Given an unlabeled input $\hat{x}\in \mathcal{X}$, i.e., a test sample, and a trained model $\mathcal{M}$, we simply compare $\phi(\hat{x})$, the \textit{query-vector}, with each class-vector and infer that the label of $\hat{x}$ is the one that corresponds to the most similar class-vector:
\begin{align*}
    \ell^{\star}(\hat{x}) = \underset{i\in\{1,\dots,k\}}{\arg\min} \;\delta\left(\phi(\hat{x}), M_i \right)
\end{align*}
where $\ell^{\star}(\hat{x}) \in \{1, \dots, k\}$ is the predicted class for $\hat{x}$.

\subsection{Regression}
\label{sec:HDCregression}

In a regression setting, the model $\mathcal{M}$ consists of a single hypervector, which memorizes training samples with their associated label. This is different from the classification setting where a sample's class is implicitly stored as the index of the class-vector. The label of each sample in a regression setting is a real number $\ell(x) \in \mathbb{R}$. To encode a label, an invertible encoding function $\phi_{\ell}$, which maps real numbers to hypervectors, needs to be introduced. The invertibility property is needed to allow labels to be determined during inference. The function's output is a finite subset $\mathbf{L} = \{L_1, \dots, L_k\}$ of all hypervectors in $\mathcal{H}$ whose generation is discussed in Section~\ref{sec:level-creation}. The hypervectors in $\mathbf{L}$ are linearly correlated such that the closer the real numbers they represent, the more similar the hypervectors are. A model is then obtained as follows:
\begin{align*}
    \mathcal{M} = \bigoplus_i \phi(x_i) \otimes \phi_{\ell}(\ell(x_i))
\end{align*}

A prediction can be made given a trained model $\mathcal{M}$ and an unlabeled input $\hat{x} \in \mathcal{X}$. First the approximate label hypervector is obtained by binding the model with the encoded sample $\mathcal{M} \otimes \phi(\hat{x}) \approx \phi_{\ell}(\ell(\hat{x}))$ which exploits the self-inverse property of binding. The remaining terms add noise, making it approximately equal~\cite{kanerva2009hyperdimensional, thomas2021theoretical}. The precise label hypervector is then the most similar label hypervector $L_l$, where:
\begin{align*}
    l = \underset{i\in\{1, \dots, k\}}{\arg\min} \;\delta\left(\mathcal{M} \otimes \phi(\hat{x}), L_i \right)
\end{align*}
Finally, the label is obtained by decoding the label hypervector using the inverse of the label encoding function:
\begin{align*}
    \ell^{\star}(\hat{x}) = \phi_{\ell}^{-1}(L_l)
\end{align*}

The encoding functions $\phi$ and $\phi_{\ell}$, used in both classification and regression, are domain specific and use the HDC operations to encode complex information (e.g., a word) by combining simpler, atomic pieces of information (e.g., the letters of a word). An example encoding for words is discussed in Section~\ref{sec:randomHVS}. The first important decision in designing an HDC encoding is how to represent atomic information as hypervectors. These hypervectors represent the smallest pieces of meaningful information and are referred to as \textit{basis-hypervectors}. These sets are the central subject in this paper. Our main goal is to show that a special set of basis-hypervectors (described in Section~\ref{sec:circularHVs}) is more appropriate for dealing with circular data.
\section{Basis-Hypervectors}
\label{sec:basisHVs}

In this section we describe the two existing basis-hypervector sets. These are stochastically created hypervector sets used to encode the smallest units of meaningful information. Their main feature is the pairwise distance as illustrated in Figure~\ref{fig:similarity-map}. This section serves as background for Section~\ref{sec:circularHVs}, where we present another basis set never before applied to learning.

\begin{figure}[h]
 \centering
    \scalebox{0.75}{\subimport*{content/images/}{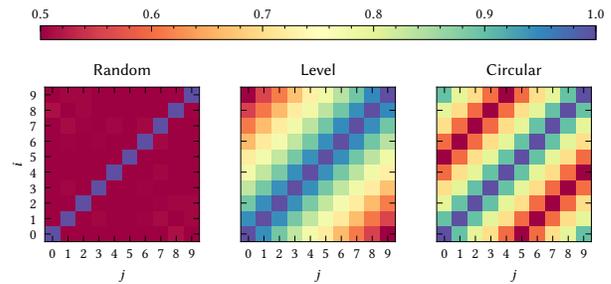}}
  \caption{\label{fig:similarity-map} Pairwise similarity of each $i$-th and $j$-th hypervector within a basis-hypervector set of size 12. Comparing between random, level and circular basis-hypervectors.} 
\end{figure}

\subsection{Encoding symbols}
\label{sec:randomHVS}

Early applications in HDC were learning from sequences of symbols such as text data~\cite{rahimi2016robust}. The units of information, in this case characters, were mapped (one-to-one) to hypervectors $\mathbf{R}=\{R_1, \dots, R_m\}$ sampled \textit{uniformly} at random from the hyperspace $\mathcal{H}$, called \textit{random-hypervectors}. From this, a word $w = \{\alpha_1 ,\dots, \alpha_n\}$ is typically encoded as:
\begin{align*}
    \phi(w) = \bigoplus_{i=1}^n \Pi^{i} \big(\phi_{\mathbf{R}}(\alpha_i )\big)
\end{align*}
where $\phi_{\mathbf{R}}(\alpha_i ) \in \mathbf{R}$ maps the character $\alpha_i$ to its corresponding random-hypervector $R_i$, and $\Pi$ is the permutation operation which ensures that information about the location of each character in the word is preserved. In general, any sequence or set comprised of symbolic or categorical data can be encoded using random-hypervectors.

Because of the high dimensionality of $\mathcal{H}$, each pair of symbols $R_i$ and $R_j$ in $\mathbf{R}$ (with $i \ne j$) is quasi-orthogonal with high probability~\cite{kanerva1988sparse}. In other words, there is no correlation between the hypervectors in the $\mathbf{R}$ basis set. While this seems suitable for encoding letters, which to some extent represent unrelated information, clearly it is not as adequate for other kinds of unitary information, such as real numbers.

\subsection{Encoding real numbers}
\label{sec:levelHVS}

Many domains use real numbers to represent information such as distance, time, energy and velocity. Representing these values in $\mathcal{H}$-space requires a mapping to a discrete set of hypervectors $\mathbf{L} = \{L_1 , \dots, L_m \}$. A real number is mapped to a hypervector with the encoding function $\phi_{\mathbf{L}}$. 

% The first step is to quantize the interval $[a,b]$ into $m$ subintervals $\{\xi_1 ,\dots, \xi_m \}$ as follows:
% \begin{align*}
%     \xi_i = \left[a+(i-1)\frac{b-a}{m}\;,\; a+i\frac{b-a}{m}\right]
% \end{align*}
% after this, each of these subintervals is represented in the hyperspace by $\phi_{\mathbf{L}}(\xi_i ) = \mathbf{l}_i $. 

Let $[a,b]$ denote the interval $\{x \in \mathbb{R} \mid a \leq x \leq b\}$ we want to represent. The first step is to place $m$ points $\{\xi_1 ,\dots, \xi_m \}$ evenly over the interval $[a, b]$, where:
\begin{align*}
\xi_i = a + (i - 1) \frac{b - a}{m - 1}
\end{align*}
Any real number $x$ is then represented in the hyperspace by $\phi_{\mathbf{L}}(x) = L_l$, where:
\begin{align*}
    l = \underset{i\in\{1,\dots,m\}}{\arg\min} \; \lvert x - \xi_i \rvert
\end{align*}

The central question here is how to construct the set of hypervectors $\mathbf{L}$. One might again consider using vectors uniformly sampled from the $\mathcal{H}$-space. Although not entirely wrong, encoding a real interval this way fails to capture an important relationship between the points on the interval. There is clearly a stronger relation between neighboring points (e.g., $\xi_1$ and $\xi_2$) when compared to $\xi_1$ and $\xi_m$, due to the different distance between these points. As indicated by our results presented in Section~\ref{sec:experiments}, encoding strategies capable of capturing such relationships (see also Section~\ref{sec:circularHVs}) lead to more powerful models. In the next section we present the way this relationship between hypervectors has thus far been achieved. In addition, we propose an improved method that we argue has more representational power.
\section{Generating level-hypervectors}
% \section{On the creation of level-hypervectors}
\label{sec:level-creation}

A method for representing real numbers with linearly correlated hypervectors was first described by Rahimi et al.~\cite{rahimi2016hyperdimensional} and Widdows and Cohen~\cite{widdows2015reasoning}. These sets are widely used in HDC and are generally referred to as \textit{level-hypervectors}. The generation of a level-hypervector set $\mathbf{L}=\{L_1, \dots, L_m \}$ starts by assigning a uniform random vector to $L_1$. Each subsequent vector is obtained by flipping $d/2/m$ bits. During the process, flipped bits are never unflipped. Therefore, the vectors $L_1$ and $L_m$, representing the endpoints of the interval, share \textit{exactly} $d/2$ bits, making them \textit{precisely} orthogonal. We argue that if the precise distance constraint is relaxed, a set with greater representation power can be created. We substantiate this claim in the next section.

% In the next section we argue that a set with more representational power can be created because the precise distances between the level-hypervectors poses a constraint on the number of level-hypervector sets that can be created which has a direct relation with the information a level-hypervector contains. By relaxing the exact distance constraint to allow the distances to be relative in expectation, making the endpoints quasi-orthogonal, more level-hypervector sets can be created which increases their information content.

% Despite their popularity, we argue that there is confusion in the literature, including conflicting descriptions, regarding the creation of these sets. 

% Clearly this strongly limits the possible outcomes of its generation which, in view of what has been discussed, is equivalent to constraining the representation power of the basis-hypervector set.

\subsection{The importance of \textit{quasi}-orthogonality}
\label{sec:imporance-orthogonality}
From an information theory perspective, the amount of information conveyed in the outcome of a random trial is a function of the probability of that outcome. More formally, for a given random variable with possible outcomes $\varepsilon_1,\dots,\varepsilon_n$, which occur with probability $\mathbb{P}(\varepsilon_1),\dots,\mathbb{P}(\varepsilon_n)$, the \textit{Shannon information content} $\mathscr{I}$ of an outcome $\varepsilon_i$ is defined as~\cite{mackay2003information}:
\begin{align*}
    \mathscr{I}(\varepsilon_i) = \log_2\frac{1}{\mathbb{P}(\varepsilon_i)}
\end{align*}

If we think of a random-hypervector set as a random sample, the probability of each realization is extremely low. Thus the entropy, or information content, is very high. This is one of the main theoretical foundations of HDC.
% One detail is very important to be highlighted at this point. Remember that, despite the intention of representing information without any correlation, the random-hypervectors are independently (and uniformly) sampled, which results in \textit{quasi}-orthogonal vectors, despite the intuition that they should be \textit{precisely} orthogonal.

Note that random-hypervectors are independently and uniformly sampled, resulting in \textit{quasi}-orthogonal vectors. These vectors are simple to create, but more importantly, they have greater representational power than \textit{precisely}-orthogonal vectors. In mathematical terms, while the number of orthogonal vectors in $\mathcal{H}$ is $d$, the number of quasi-orthogonal vectors is almost $2^d$~\cite{kanerva1988sparse}. Given that each set is sampled uniformly, a much larger sample space implies a much lower probability of occurrence per outcome. By the definition above, this results in greater information content.

\subsection{Level-hypervectors revisited: applying the notion of ``quasi''}

% \subsection{Applying the notion of \textit{quasi}-orthogonality\igor{a little imprecise: need to mention its to level-hvs and not just about the endpoints}}

% Building on the discussion just presented, in this section we argue that the creation of level-hypervectors, as widely used in the HDC literature, can be improved. First, note that restricting the basis set to correlated hypervectors already lowers the information content of the set. We address this issue and propose an improvement in Section~\ref{sec:IncreasingRandomness}. 

The existing method for generating a level-hypervector set does not incorporate an important property of the random-hypervector set. As we discussed in the previous section, the key to the representational power of random-hypervectors---and more generally of HDC---comes from the relaxed notion of distance: quasi-orthogonality.
% Another problem is that the existing method for generating a level-hypervector set does not incorporate an important property of random-hypervectors. As we discussed in the previous section, the key to the representational power of random-hypervectors---and more generally of HDC---comes from the relaxed notion of distance, quasi-orthogonality.
% Another problem is that, although initially inspired by and proposed as a generalization of random-hypervectors, the existing method for generating level-hypervectors fails to reproduce one of their most important aspects. As we discussed in the previous section, the key to the representational power of random-hypervectors --- and more generally of HDC --- comes from the more relaxed notion of distance (quasi-orthogonality). % In contrast, the methodology currently used to generate level-hypervectors $\mathbf{L}=\{L_1, \dots, L_m \}$ starts assigning a uniform random vector to $L_1$ and each subsequent vector is \textit{exactly} $d/2/m$ distant from the previous one. In particular, the vectors $L_1$ and $L_m$, representing the endpoints of the interval are \textit{precisely} orthogonal.%
In contrast, the level-hypervectors created with the existing method, as described above, have a fixed distance between each pair of hypervectors. This limits the number of possible outcomes of their generation which, in view of what has been discussed, is equivalent to constraining the representation power of the basis-hypervector set. 
Instead, we want the distance between two hypervectors $L_i$ and $L_j$ in $\mathbf{L}$, with $i < j$, to be proportional to $j-i$ \textit{in expectation}. If we denote by $\Delta^{i,j}$ the desired value for $\mathbb{E}\left[\delta(L_i, L_j)\right]$, then:
% Consider two hypervectors $L_i$ and $L_j$ in $\mathbf{L}$, with $i < j$. These represent $\xi_i$ and $\xi_j$, respectively, from a list of $m$ equidistant points, as was presented in Section~\ref{sec:levelHVS}. Thus, to capture their linear correlation, we want the distance $\rho(L_i, L_j)$ to be proportional to $j-i$. In addtion, we want the endpoints to be orthogonal $\rho(L_1, L_m) = d/2$ (quasi-orthogonal). Simply put, if we denote by $\Delta^{i,j}$ the expected distance between $L_i$ and $L_j$, then: 
% Consider two vectors $L_i$ and $L_j$ in $\mathbf{L}$, with $i < j$. From what was presented in Section~\ref{sec:levelHVS}, remember that these represent, respectively, $\xi_i$ and $\xi_j$, from a list of $m$ equidistant points $\{\xi_1,\dots,\xi_m\}$. For this reason, we want the distance $\rho(L_i, L_j)$ to be proportional (but not precisely) to $j-i$, and also $\rho(L_1, L_m) \approx d/2$ (quasi-orthogonal). Simply put, if we denote by $\Delta^{i,j}$ the expected distance between $L_i$ and $L_j$, then: 
\begin{align*}
    \Delta^{i,j} = \frac{j-i}{2(m-1)}
\end{align*}

% Remember that in the existing approach to generate $L_j$, we simply flip $\Delta^{i,j}$ bits of $L_i$ in \textit{different} positions. The problem is that solution results in $\rho(L_i, L_j) = \Delta^{i,j}$ and, for the reasons mentioned above, a more relaxed distance relationship --- $\rho(L_i, L_j) \approx \Delta^{i,j}$ --- is crucial.

% An intuitive idea to solve this problem, that is, make them just ``nearly" $\Delta^{i,j}$-distant, is to allow multiple flips per position and perform a number of flips $\widetilde{\Delta}^{i,j}$ that results in the \textit{expected value} of the distance to be $\Delta^{i,j}$. Notice that, as in random-hypervectors, $L_1$ and $L_m$  would just be \textit{expected to be} orthogonal. It remains to be determined, therefore, such a value $\widetilde{\Delta}^{i,j}$. 

An intuitive idea, that we describe to clarify the problem, is to allow multiple flips per position and perform a number of flips $\digamma^{i,j}$ that results in an expected distance $\Delta^{i,j}$. One could use a Markov Chain as illustrated in Figure~\ref{fig:markov-chain} to model the bit flipping process. Each state $S(t)$ represents a distance $\delta$ from $L_i$. The process always starts from $S(0) = 0$. At each following step one bit is flipped at a random position, moving either one bit away from $L_i$ (represented by solid arrows) or go back one bit (dashed arrows). The probability of moving away from $L_i$ in $S(t)$ is:
\begin{align*}
    \mathbb{P}\left[S(t+1) = S(t) + 1\right] = \frac{d - S(t)}{d}
\end{align*}
Observe that $\digamma^{i,j}$ is the expected number of steps taken until absorption in the $\Delta^{i,j}$ state.

\begin{figure}
    \centering
    \begin{tikzpicture}[auto,node distance=15mm, thick,main node/.style={circle, draw,minimum size=.4cm,inner sep=0pt, minimum height=.6cm}]
    \tikzset{every node}=[font=\footnotesize]

    \node[main node] (1) {$0$};
    \node[main node] (2) [right of=1] {$1/d$};
    \node[main node] (3) [right of=2] {$2/d$};
    \node[minimum width = 1.3cm] (4) [right of=3] {$\cdots$};
    \node[main node, ellipse] (5) [right of=4] {$\Delta^{i,j} - \frac{1}{d}$};
    \node[main node, ellipse] (6) [right of=5] {$\Delta^{i,j}$};
    
    \path
    (1) edge[-latex,bend left] node{} (2)
    (2) edge[-latex,bend left] node{} (3)
    (3) edge[-latex,bend left] node{} (4)
    (4) edge[-latex,bend left,in=160] node[anchor=center, midway, above]{} (5)
    (5) edge[-latex,bend left] node{} (6)
    
    (2) edge[-latex,bend left, dashed] node{} (1)
    (3) edge[-latex,bend left, dashed] node{} (2)
    (4) edge[-latex,bend left, dashed] node{} (3)
    (5) edge[-latex,bend left, dashed] node{} (4);
\end{tikzpicture}
    \caption{Markov Chain of a bit flipping process}
    \label{fig:markov-chain}
\end{figure}
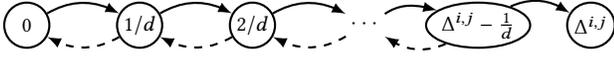

Let $u(k)$ denote the expected absorption time starting from state $k$. By considering each starting state and conditioning on each possible outcome, we get the following linear recurrence relation:
\begin{align*}
    u(k) = 
    \begin{cases}
    1+u(1) &\;\text{ if }\;k=0\\
    1+\frac{(d-k)u(k+1)+ku(k-1)}{d} &\;\text{ if }\;0 < k < \Delta^{i,j}\\
    0 &\;\text{ if }\;k=\Delta^{i,j}
    \end{cases}
\end{align*}
For a given dimension $d$, this is a solvable tridiagonal linear system~\cite{stone1973efficient}. The number of bits $\digamma^{i,j}$ we need to invert to obtain $L_j$ from $L_i$, such that $\mathbb{E}\left[\delta(L_i, L_j)\right] = \Delta^{i,j}$ is exactly $u(0)$, obtained from solving the system. This method generates what are called scatter codes~\cite{smith1990random} which map the input space nonlinearly to the similarities in $\mathcal{H}$-space. In the next section we propose a practical algorithm, simple to implement and use in HDC applications with a linear mapping.
% Although this analysis provides useful insight into the creation of level-hypervectors, in the next section we propose a practical algorithm, simple to implement and use in HDC applications.

\subsection{A new method for generating level-hypervectors}
\label{sec:newMethod}

Presented below, Algorithm~\ref{alg:level_hvs} starts by assigning two uniformly random hypervectors to $L_1$ and $L_m$, in addition to a $d$-dimensional vector $\Phi$ whose elements are sampled uniformly from [0,1]. Then, for each remaining level $L_l$ an interpolation threshold value $\tau_l$ is set and $\Phi$ acts as a filter to determine the origin of each bit, either from $L_1$ or $L_m$. This differs from the method by~\citet{rachkovskiy2005sparse}, which deterministically concatenates fractions of $L_1$ and $L_m$. Proposition~\ref{prop:level_hvs} establishes that the obtained set of level-hypervectors meets the previously motivated property, that is, for any $j>i$, the expected distance between any $L_i$ and $L_j$, is $\Delta^{i,j}$.

\begin{algorithm}
    \SetKwInOut{Input}{Input}
    \SetKwInOut{Output}{Output}

    \Input{Two integers $m$ and $d$.}
    \Output{A set of $m$ $d$-dimensional level-hypervectors $\mathbf{L} = \{L_1, \ldots, L_m\}$}
    $L_1 \gets$ uniform\_sample($\{0,1\}^d$)\\
    $L_m \gets$ uniform\_sample($\{0,1\}^d$)\\
    $\Phi \gets$ uniform\_sample($[0,1]^d$)\\
    
    \For{each remaining level $l \in \{2, \ldots, m-1\}$}
    {
        $\tau_l \gets \frac{m - l}{m - 1}$\\
        \For{each dimension $\partial \in \{1, \dots, d\}$}
        {
            \If{$\Phi^{(\partial)} < \tau_l$}{
                $L_l^{(\partial)} \gets L_1^{(\partial)}$
            }
            \Else{
                $L_l^{(\partial)} \gets L_m^{(\partial)}$ 
            }
        }
    }

    \Return $\{L_1, \ldots, L_m\}$
    \caption{Creation of a level-hypervector set using interpolation filters.}
    \label{alg:level_hvs}
\end{algorithm}

\begin{proposition}[]
\label{prop:level_hvs}
Let $\mathbf{L} = \{L_1 ,\dots\,L_m \}$ denote a set of hypervectors generated by Algorithm~\ref{alg:level_hvs}. For all $i$ and $j>i$ in $\{1,\dots,m\}$, we have $\mathbb{E}\left[\delta(L_i, L_j)\right] = \Delta^{i,j}$.
\label{thm:alg1}
\end{proposition}
\begin{proof}
First, from the definition of the normalized Hamming distance, we have:
\begin{align*}
    \delta \left(L_i, L_j\right) = \frac{1}{d}\sum_{\partial=1}^d \mathbbm{1}\left(L_i^{\partial} \ne L_j^{\partial}\right)
\end{align*}
where $\mathbbm{1}$ is the indicator function. By applying the linearity of expectation property, the i.i.d. property for all dimensions of $L_i$, and considering that the expectation of an indicator function equals the probability of the event, we get:
% \begin{align}
%     \begin{split}
%     \mathbb{E}\left[\delta(L_i, L_j)\right] = \frac{1}{d}\sum_{\partial=1}^d \mathbb{E}\left[\mathbbm{1}\left(L_i^{\partial} \ne L_j^{\partial}\right)\right] = 1 - \mathbb{P}\left(L_i^{\bar{\partial}} = L_j^{\bar{\partial}}\right)
%     \end{split}
% \end{align}
\begin{align}
    \mathbb{E}\left[\delta(L_i, L_j)\right] =  \mathbb{P}\left(L_i^{\bar{\partial}} \ne L_j^{\bar{\partial}}\right)
\end{align}
where $\bar{\partial} \in \{1,\dots,d\}$ indicates that the probability is dimension independent. Then, from Algorithm~\ref{alg:level_hvs} we have:
\begin{align*}
    \mathbb{P}(L_i^{\bar{\partial}}=L_j^{\bar{\partial}}) =
    \mathbb{P}(\Phi^{\bar{\partial}}<\tau_{i})\left(\genfrac{}{}{0pt}{}{\mathbb{P}(\Phi^{\bar{\partial}}<\tau_{j}|\Phi^{\bar{\partial}}<\tau_{i})\mathbb{P}(L_1^{\bar{\partial}}=L_1^{\bar{\partial}})}{+\mathbb{P}(\Phi^{\bar{\partial}}\geq\tau_{j}|\Phi^{\bar{\partial}}<\tau_{i})\mathbb{P}(L_1^{\bar{\partial}}=L_m^{\bar{\partial}})
    }\right)\\
    +\mathbb{P}(\Phi^{\bar{\partial}}\geq\tau_{i})\left(\genfrac{}{}{0pt}{}{\mathbb{P}(\Phi^{\bar{\partial}}\geq\tau_{j}|\Phi^{\bar{\partial}}\geq\tau_{i})\mathbb{P}(L_1^{\bar{\partial}}=L_1^{\bar{\partial}})}{+\mathbb{P}(\Phi^{\bar{\partial}}<\tau_{j}|\Phi^{\bar{\partial}}\geq\tau_{i})\mathbb{P}(L_1^{\bar{\partial}}=L_m^{\bar{\partial}})}\right)
\end{align*}
Given that $\Phi^{\bar{\partial}}$ is uniform in $[0,1]$ and $\tau_i=\frac{m-i}{m-1}$ according to the algorithm, we can calculate this probability to be:
\begin{align}
    \mathbb{P}\left(L_i^{\bar{\partial}}=L_j^{\bar{\partial}}\right) = 1 - \frac{j-i}{2(m-1)}
\end{align}
Considering that the event is binary, from Equations 1 and 2, we get: $\mathbb{E}\left[\delta\left(L_i, L_j\right)\right] = \Delta^{i,j}$.
\end{proof}

\section{Encoding Circular Data}
\label{sec:circularHVs}

\label{sec:circularData}

As described above, symbols and real numbers can be represented in the hyperspace with random and level-hypervector basis sets, respectively. However, not every type of data falls into these two categories. Consider, for instance, angular data in $\Theta=[0,2\pi]$. The distance $\rho \in [0,1]$ between two angles $\alpha$ and $\beta$ in $\Theta$ is defined as~\cite{lund1999least}:
\begin{align*}
    \rho(\alpha,\beta) = \frac{1}{2}\left(1-\cos(\alpha-\beta)\right)
\end{align*}
If we use level-hypervectors to encode the $\Theta$-interval, the distances between the hypervectors would not be proportional to the distance between the angles. Notice that the endpoints of an interval represented with level-hypervectors are completely dissimilar, while a circle has no endpoints.

Angles are widely used to represent information in meteorology~\cite{holzmann2006hidden}, ecology~\cite{tracey2005set,kempter2012quantifying}, medicine~\cite{gao2006application,gao2014jhu,ahmidi2017dataset}, astronomy~\cite{cabella2009statistical,marinucci2011random} and engineering~\cite{chirikjian2000engineering}. Moreover, many natural and social phenomena have circular-linear correlation on some time scale. Consider for example the seasonal temperature variations over a year or the behavior of fish with respect to the tides in a day. In these cases, it makes sense to represent the time intervals (e.g., Jan 1st - Dec 31st) as cyclic intervals~\cite{lund1999least,kempter2012quantifying}.

Given the multitude of applications using circular data, unsurprisingly there has been great scientific effort to adapt statistical and learning methodologies to handle it appropriately~\cite{lee2010circular}. This gave rise to a branch of statistical methodology known as \textit{directional statistics}~\cite{mardia2000directional,pewsey2021recent}. Despite all this effort, to the best of our knowledge, our work is the first to address the adaptation in the context of HDC learning. 

\subsection{Circular-hypervectors}
\label{sec:creatingCircularHVs}

An algorithm for creating a set of hypervectors whose distances are proportional to that of a set of equidistant points on a circle has recently been proposed by Heddes et al.~\cite{heddes2022hyperdimensional}. They used the set, called \textit{circular-hypervectors}, to create a dynamic hashing system, where the circular structure helps with evenly distributing requests across a changing population of servers~\cite{karger1997consistent}. We propose to adapt and generalize their algorithm to create a basis set of hypervectors suitable for learning from circular data using HDC.

We want to build a set of hypervectors $\mathbf{C}=\{C_1, \dots, C_m\}$\footnote{We assume $m$ to be even to simplify discussions. Sets of odd cardinality can be generated as subsets $\{C_1, C_3, C_5, \dots, C_{2m - 1} \}$ of a set of size $2m$.} such that for all $C_i$ and $C_j$ in $\mathbf{C}$ their distance $\delta$ in $\mathcal{H}$ is proportional to the distance between the angles they represent:
\begin{align*}
    \mathbb{E}\left[\delta(C_i, C_j)\right] = \frac{1}{2}\,\rho\left((i-1)\frac{2\pi}{m},(j-1)\frac{2\pi}{m}\right)
\end{align*}
This relationship is in terms of expected value to improve the information content as discussed in Section~\ref{sec:level-creation}.

\begin{figure}[H]
\centering
\begin{tikzpicture}[auto, thick, main node/.style={draw, minimum width = 1.2cm, minimum height = .7cm}]
    \tikzset{every node}=[font=\footnotesize\sffamily]
    
    \fill [rounded corners=1ex, gray!10](-0.9, -0.65) rectangle (7.0, 1.0);
    \node[] (3) at (3.1, 0.7) {Phase 1};
    
    \fill [rounded corners=1ex, gray!10](-0.9, -0.85) rectangle (7.0, -2.5);
    \node[] (3) at (3.1, -1.15) {Phase 2};

    % first half
    \node[main node] (1) {$C_1$};
    \node[main node] (2) [right=1cm of 1] {$C_2$};
    \node[] (3) [right=1cm of 2] {$\cdots$};
    \node[main node] (4) [right=1cm of 3] {$C_{m/2 + 1}$};
    
    \path
    (1) edge[-] node[above]{$T_1$} (2)
    (2) edge[-] node[above]{$T_2$} (3)
    (3) edge[-] node[above]{$T_{m/2}$} (4)
    ;
    
    % second half
    \node[main node] (5) [below=1.15cm of 4] {$C_{m/2 + 2}$};
    \node[main node] (6) [left=1cm of 5] {$C_{m/2 + 3}$};
    \node[] (7) [left=1cm of 6] {$\cdots$};x
    \node[main node] (8) [left=1cm of 7] {$C_{m}$};
    
    \path
    (4) edge[-latex] node[right, yshift=-2ex]{$T_1$} (5)
    (5) edge[-latex] node[above]{$T_2$} (6)
    (6) edge[-latex] node[above]{$T_3$} (7)
    (7) edge[-latex] node[above]{$T_{m/2 - 1}$} (8)
    (8) edge[-latex, dashed] node[left, yshift=-2ex]{$T_{m/2}$} (1)
    ;

\end{tikzpicture}
\caption{\label{fig:create-circle-hv} Diagram of the creation of circular-hypervectors. Phase 1 shows the level-hypervectors and the transformations between them. Phase 2 shows the use of transformations for the second half of the circle. The dashed transformation is shown to complete the circle but is not needed since $C_1$ is already known.} 
\end{figure}
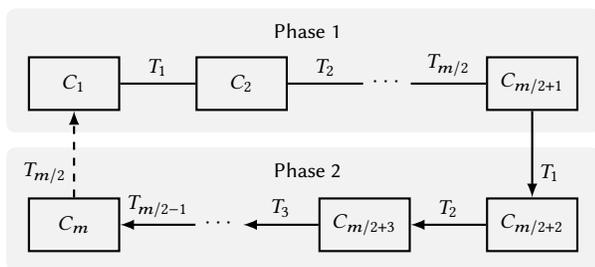

The creation of circular-hypervectors, shown in Figure~\ref{fig:create-circle-hv}, is divided into two phases, one for each half of the circle. The first half is simply a set of $m/2 + 1$ level-hypervectors: 
\begin{align*}
    C_i = L_i \text{ for } i \in \left\{1, \dots, m/2 + 1\right\}
\end{align*}
where $C_1$ and $C_{m/2 + 1}$ are quasi-orthogonal, ensuring that the opposite side of the circle is completely dissimilar. The second half is created by applying the transitions between the levels of the first half, in order, to the last circular-hypervector: 
\begin{align} \label{eq:circle-phase2}
    C_i = C_{i-1} \otimes T_{i - m/2 - 1}, \text{ for } i \in \left\{m/2 + 2, \dots, m\right\}
\end{align}
where the transition $T_i = C_{i} \otimes C_{i + 1}$ are the flipped bits between levels $i$ and $i + 1$. The combined transitions $\{T_1, \dots, T_{m/2}\}$ are equal to the transition from $C_1$ to $C_{m/2 + 1}$ such that:
\begin{align*}
    C_{m/2 + 1} = C_{1} \otimes \bigotimes_{i = 1}^{m/2} T_i
\end{align*}
Since binding is its own inverse, the transitions bound to $C_{i-1}$ in Equation~\ref{eq:circle-phase2} make the new hypervector $C_{i}$ closer to $C_1$. Moreover, the transitions $\{T_1, \dots, T_{m/2}\}$ occur in any half of the circle, ensuring that the hypervector at the opposite side from any point is always quasi-orthogonal to it.

% Locally, however, the transition ensures that the new hypervector is some distance away from the last hypervector. \mike{Add stronger claim that the transformation ensure that the new hypervector is some distance further from all the hypervectors in the last half circle. Also, look for theory in linear algebra that can explain this.}

% \subsection{Adjusting randomness}
\subsection{Controlling the trade-off between correlation preservation and information content}
\label{sec:IncreasingRandomness}
% The discussion in Section~\ref{sec:imporance-orthogonality} illustrated the benefit of relaxing the restriction of exact orthogonality, resulting in greater information content. Another restriction to the level and circular-hypervector sets thus far has been the endpoints, which are set to be the only pair within the set that is orthogonal. We argue that more powerful models can be created if this constraint is lifted as well. The random basis set is very capable at representation as it is the least restricted set. However, it is lacking the ability to generalize precisely because every hypervector is quasi-orthogonal, that is, no relative correlation is preserved. The level and circular sets preserve relative correlation but are restricted in information content. 

The discussion in Section~\ref{sec:imporance-orthogonality} illustrated the benefit of relaxing the distance between generated hypervectors, resulting in greater information content. Another concern is that, while level and circular-hypervectors have the ability to translate important correlations into hyperspace, from the perspective of information content this diminishes their representational power: by forcing the vectors to be correlated, the probability distribution over the possible outcomes becomes more concentrated, decreasing the entropy.
We argue that more powerful models can be created if this constraint, i.e., the correlation, is relaxed as well.

The random-hypervectors are the most capable at representation as they are sampled uniformly, without any restriction. However, for this very reason they are unable to map existing correlations in the input space to the hyperspace. The level and circular sets preserve correlation but are restricted in information content. The ideal basis set is then expected to have a balance between its ability to preserve correlation and its information content.

% To address this trade-off, we introduce a hyperparameter $r \in [0, 1]$ that interpolates between a level or circular-hypervector set, with quasi-orthogonal endpoints, and the random set where every pair of vectors is quasi-orthogonal. 
To address this trade-off, we introduce a hyperparameter $r \in [0, 1]$ that interpolates between a level or circular-hypervector set, and the random set. As can be seen in Figure~\ref{fig:r-value-circle}, the $r$-hyperparameter changes the amount of correlation between neighboring nodes. Intermediate values of $r$ thus enable higher information content while locally the relative correlation is preserved.

\begin{figure}[H]
\centering
\input{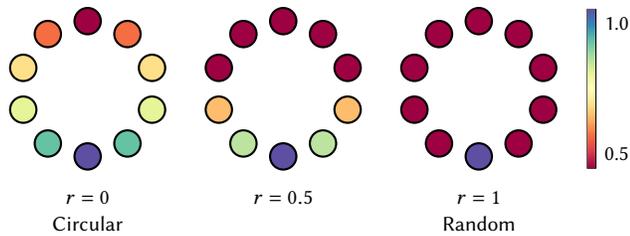}
\caption{\label{fig:r-value-circle} Effect of the $r$-hyperparameter on the similarities between each node and the bottom reference node in a circular set of 10 hypervectors.}
\end{figure}

To interpolate we concatenate multiple level-hypervector sets created with Algorithm~\ref{alg:level_hvs}. The last hypervector of one set becomes the first hypervector of the next set. The number of transitions $n$ per level-hypervector set is given by:
\begin{align*}
n = r + (1 - r)(m - 1)    
\end{align*}
where $m$ is the total number of hypervectors in the concatenated set. Each level in the concatenated set is obtained by using the threshold value:
\begin{align*}
 \tau_l = 1 - \frac{(l - 1)\; \mathrm{mod}\; n}{n}   
\end{align*}
When $r=1$, each level set contains only two hypervectors, (i.e., one transition) resulting in a set identical to a random-hypervector set. For circular-hypervectors the change only applies to phase 1 (see Section~\ref{sec:creatingCircularHVs}). 

%  The number of randomly sampled hypervectors needed to perform the stacking is $\ceil*{\frac{(m - 1)}{n} + 1}$ between each random hypervector a filter vector $\Theta$ is sampled. The threshold value $\tau$ is $1 - \frac{l \mod n}{n}$ and the level-set index is obtained by $\floor*{\frac{l}{n}}$ Then for each level get the boundary random hypervectors and the filter vector and use tau to interpolate between the boundaries. and a filter vector $\Theta$ and boundary random-hypervectors $L_1$ and $L_m$ of the $i$-th concatenated set, where $i = \floor*{(l - 1)/n} + 1$.

\section{Experiments}
\label{sec:experiments}

We evaluate circular-hypervectors in classification and regression settings where they are compared to random and level-hypervectors. The $r$-hyperparameter is evaluated by observing its effect on the same two settings.

\subsection{Classification}
\label{sec:classificationResults}
To evaluate the performance of circular-hypervectors in a classification setting, we use the JHU-ISI Gesture and Skill Assessment Working Set (JIGSAWS) dataset~\cite{gao2014jhu,ahmidi2017dataset}.
% --- a public dataset containing data from surgeons operating the \textit{da Vinci} robot while performing surgical tasks.
The dataset contains videos accompanied by 76-dimensional kinematic data of three different surgical tasks, performed by eight different surgeons, operating the \textit{da Vinci} robot. As our goal is to show whether circular-hypervectors are more appropriate for dealing with circular data, we use a subset of the dataset containing circular data. Specifically, we use the 18 kinematic variables that represent the rotation matrices of the left master tool manipulator and patient-side manipulator.

Each sample has a label which indicates the surgical activity called ``gestures.'' There are 15 gestures, each representing intentional surgical actions such as ``orienting needle,'' ``pushing needle through tissue'' and ``transferring needle from left to right.'' The task is to classify which gesture is being performed. For this, we used the standard classification framework presented in Section~\ref{sec:HDCclassification}.

A sample is encoded as $\textstyle \bigoplus_{i = 1}^{18} K_i \otimes V_i$ where the key $K_i$ is a random-hypervector that represents the index $i$, and the value $V_i$ is encoded as a random, level or circular-hypervector, depending on the experiment. For each of the three surgical tasks, \textit{Knot Tying}, \textit{Needle Passing} and \textit{Suturing}, a model was trained on the data from surgeon "D", one of the most experienced in the experiment.

\begin{table}
  \caption{Classification accuracy results. The circular hypervectors have $r=0.1$.}
  \label{tab:classification-acc}
  \begin{tabular}{lccc}
    \toprule
    Dataset & Random & Level & Circular\\
    \midrule
    Knot Tying & 76.6\% & 75.9\% & \textbf{84.0\%}\\
    Needle Passing & 76.0\% & 76.0\% & \textbf{83.6\%}\\
    Suturing & 73.0\% & 60.4\% & \textbf{78.7\%}\\
  \bottomrule
\end{tabular}
\end{table}

The results comparing the classification accuracy of each basis-hypervector set for each surgical task are shown in Table~\ref{tab:classification-acc}. The circular-hypervectors perform an average of 7.2\% better than random-hypervectors, which in-turn slightly outperform level-hypervectors. The training and evaluation running time are nearly equivalent among all basis sets because the one-time differentiating cost of generating the basis set is negligible compared to the training time.

\subsection{Regression}
\label{sec:regressionResults}

In addition to classification, we evaluate the performance of circular-hypervectors on two regression tasks. The first contains temperature data measured in Beijing, from March 2013 till February 2017 at the Aotizhongxin station~\cite{zhang2017cautionary}, available on the UCI Machine Learning Repository~\cite{Dua:2019}. We hypothesize that the circular-hypervectors are more appropriate for representing the day of the year and the hour of the day in this setting because they are both proxies of angular values: the angle of the earth in its orbit around the sun and the earth around its axis, respectively. Since both of these values are highly correlated with the outside temperature, we expect an HDC model that preserves the angular correlation to be better at forecasting the temperature.

Each model was trained on the first 70\% of the data using the regression framework described in Section~\ref{sec:HDCregression}. The samples are encoded as $Y \otimes D \otimes H$ with the year $Y$, day $D$, and hour $H$ hypervectors. The year hypervector is a level-hypervector, enabling the model to capture macro trends in temperature, such as global warming. The day and hour hypervectors change between random, level, and circular depending on the experiment. The label is the temperature at a given time, encoded as a level-hypervector. After the models are trained, one for each type of basis-hypervector, their performance is measured by calculating the mean squared error on predictions of the last 30\% of the dataset.

The second dataset contains power level measurements from the Mars Express satellite, provided by the European Space Agency~\cite{dressler2019mars}. The power level of the satellite fluctuates based on its orbit and the on-board energy consumption~\cite{lucas2017machine}. A more accurate model for available power means that the safety margins can be reduced, therefore allowing more energy to be dedicated towards scientific operations.

A training sample consists of the elapsed fraction of Mars' orbit around the sun, called the mean anomaly. The encoding of the mean anomaly changes based on the experiment between random, level, and circular-hypervectors. The label is the power level at a given time, encoded as a level-hypervector. The data is randomly split between 70\% training and 30\% testing. The performance of each basis set is determined on the test split using the mean squared error metric.

\begin{figure}[h]
 \centering
    \scalebox{0.70}{\subimport*{content/images/}{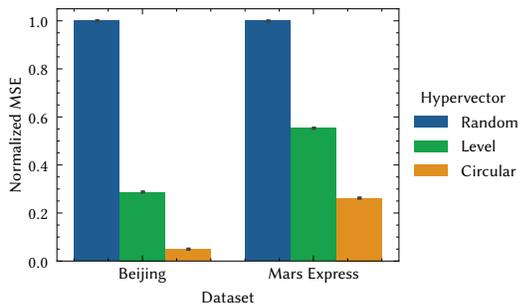}}
    \caption{\label{fig:regression-mse} Regression normalized mean squared error results for each basis hypervector type. Normalized against the performance of random hypervectors. The circular hypervectors have $r=0.01$.} 
\end{figure}

The results for both regression tasks are presented in Table~\ref{tab:regression-mse}. Circular-hypervectors reduce the error by 67.7\% and 84.4\% on average compared to level-hypervectors and random-hypervectors, respectively. These results, in combination with those for classification, illustrate how domain knowledge about the circular nature of the data can account for a significant improvement in performance. The efficiency of HDC ensures that most embedded systems can afford HDC models within their computation budget. This makes circular-hypervectors ideally suited for embedded and IoT application that are rich in circular data, such as robotics, where actuators and joints generate much angular data.
 
\begin{table}
  \caption{Regression mean squared error results. The circular hypervectors have $r=0.01$.}
  \label{tab:regression-mse}
  \begin{tabular}{lccc}
    \toprule
    Dataset & Random & Level & Circular\\
    \midrule
    Beijing & 441.1 & 126.8 & \textbf{21.9}\\
    Mars Express & 1294.1 & 715.6 & \textbf{339.1}\\
  \bottomrule
\end{tabular}
\end{table}

\subsection{R-value}
\label{sec:exp-r-value}

We evaluate the $r$-hyperparameter by observing the performance of the classification and regression tasks described above as we vary the $r$-value between 0 and 1, interpolating between circular and random-hypervectors. In order to aggregate the effects in one plot we use the normalized mean squared error metric for the regression tasks and the normalized accuracy error, defined as $\frac{1 - \alpha}{1 - \bar{\alpha}}$ where $\alpha$ is the accuracy and $\bar{\alpha}$ the reference accuracy. The reference for all tasks is set to the performance of random-hypervectors, equivalent to the endpoint of the interpolation.

\begin{figure}[h]
 \centering
 \scalebox{0.7}{\subimport*{content/images/}{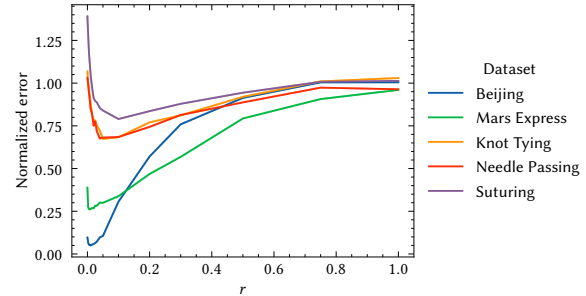}}
  \caption{\label{fig:r-value} Error of the circular-hypervector basis set with varying $r$-hyperparameter, normalized against the random-hypervector set performance per dataset.} 
\end{figure}

The normalized error plot shown in Figure~\ref{fig:r-value} indicates that better performance can be achieved when $r > 0$, as is inline with the theoretical analysis presented in Section~\ref{sec:imporance-orthogonality}. After improving initially, as $r$ gets closer to 1, the performance approaches that of the random-hypervectors. These results indicate the importance of considering the information content of a basis-hypervector set as we propose in our  work. In addition, it shows the value of the proposed $r$-parameter to control the trade-off between representation power and the ability to preserve correlations in the hyperspace mapping.
\section{Conclusion}
\label{sec:conclusion}
We study basis-hypervectors: stochastically created vectors used to represent atomic information in HDC. Taking inspiration from information theory, we propose a methodology to create level-hypervectors with greater representational power. Furthermore, we introduce a method to handle circular data in HDC. This method, which uses the improved level-hypervectors, is the first approach to learning from circular data in HDC. We believe that these contributions have the potential to benefit HDC in general, as they improve the accuracy of models based on circular and real data, present in most learning applications.

\clearpage
\bibliographystyle{ACM-Reference-Format}
\bibliography{references}

\end{document}